\documentclass[12pt]{article}
\usepackage[left=1in,top=1in,right=1in,nohead,bottom=1in]{geometry}
\usepackage{graphicx}
\usepackage{amsmath, amssymb, enumerate}
\usepackage[applemac]{inputenc}
\usepackage[T1]{fontenc}
\usepackage{amsfonts}
\usepackage{tikz}
\usetikzlibrary{shapes.geometric, arrows}

\usepackage{parskip}
\usepackage{amsthm}
\usepackage{natbib}
\newtheorem{theorem}{Theorem}
\newtheorem*{definition}{Definition}

\newtheorem{lemma}{Lemma}[]
\newcommand{\ds}{\displaystyle}
\newcommand{\E}{\text{E}}

\newcommand{\ts}{\textsuperscript}
\DeclareMathOperator*{\argmax}{\arg\!\max}

\title{Randomized Allocation with Nonparametric Estimation for Contextual Multi-Armed Bandits with Delayed Rewards} 

\author{Sakshi Arya and Yuhong Yang\\
\small{School of Statistics, University of Minnesota} }
\date{}  

\begin{document}
\maketitle







\begin{abstract}
We study a multi-armed bandit problem with covariates in a setting where there is a possible delay in observing the rewards.  Under some reasonable assumptions on the probability distributions for the delays and using an appropriate randomization to select the arms, the proposed strategy is shown to be strongly consistent.
\end{abstract}




\section{Introduction}

 Multi-armed bandits were first introduced in the landmark paper by \cite{robbins1952}. The development of multi-armed bandit methodology has been partly motivated by clinical trials with the aim of balancing two competing goals, 1) to effectively identify the best treatment (exploration) and 2) to treat patients as effectively as possible during the trial (exploitation).

The classic formulation of the multi-armed bandit problem in the context of clinical practice is as follows: there are $\ell \geq 2$ treatments (arms) to treat a disease. The doctor (decision maker) has to choose for each patient, one of the $\ell$ available treatments, which result in a reward (response) of improvement in the condition of the patient. The goal is to maximize the cumulated rewards as much as possible. 
In the classic multi-armed bandit terminology, this is achieved by devising a policy for sequentially pulling arms out of the $\ell$ available arms, with the goal of maximizing the total cumulative reward, or minimizing the regret. Substantial amount of work has been done both on standard context-free bandit problems (\cite{gittins1979bandit}, \cite{berry1985bandit}, \cite{lai1985asymptotically},  \cite{auer2002finite}) and on 
contextual bandits or multi-armed bandits with covariates (MABC) (\cite{woodroofe1979one}, \cite{sarkar1991one}, \cite{yang2002randomized}, \cite{langford2008epoch}, \cite{li2010contextual}  and \cite{slivkins2014contextual}). The MABC problems have been studied in both parametric and nonparametric frameworks. Our work follows nonparametric framework of MABC in \cite{yang2002randomized} where the randomized strategy is an annealed $\epsilon$-greedy strategy, which is a popular heuristic in bandits literature (\cite{sutton2018reinforcement}, Chapter 2). Some of the other notable work in studying finite time analysis for MABC problems in a nonparametric framework are \cite{perchet2013multi,qian2016kernel,qian2016randomized}. Some insightful overviews and bibliographic remarks can be found in \cite{bubeck2012regret}, \cite{cesa2006prediction}, \cite{lattimore2018bandit}.

 In most multi-armed bandit settings it is assumed that the rewards related to each treatment allocation are achieved before the next patient arrives. This is not realistic since in most cases the treatment effect is seen at some delayed time after the treatment is provided. Most often, it would be the case that while waiting for treatment results of one patient, other patients would have to be treated. In such a situation, all past patient information and feedback is not yet available to make the best treatment choices for the patients being treated at present.

  While an overwhelming amount of work has been done assuming instantaneous observations in both contextual and non-contextual multi-armed bandit problems, not much work has been done for the case with delayed rewards.  
The importance of considering delays was highlighted by \cite{anderson1964sequential} and \cite{suzuki1966sequential}. They used Bayesian multi-armed bandits to devise optimal policies. Thompson sampling (\cite{agrawal2012analysis,russo2018tutorial}) is another commonly used Bayesian heuristic. \cite{chapelle2011empirical} conducted an empirical study to illustrate robustness of Thompson sampling in the case of constant delayed feedback. Most of the work that has been done in the recent years is motivated by reward delays in online settings like advertisement and news article recommendations. \cite{Dudik:2011:EOL:3020548.3020569} considered a constant known delay which resulted in an additional additive penalty in the regret for the setting with covariates.  \cite{joulani2013online} propose some black box multi-armed bandit algorithms that use the algorithms for the non-delayed case to handle the delayed case. Their finite time results show an additive increase in the regret for stochastic multi-armed bandit problems.  
More recently, \cite{pike2018bandits} proposed a variant of delayed bandits with aggregated anonymous feedback. They show that with their proposed algorithm and with the knowledge of the expected delay, an additive regret increase like in \cite{joulani2013online} can still be maintained. Some other work related to delayed bandits can be found in \cite{mandel2015queue}, \cite{cesa2016delay} and \cite{vernade2017stochastic}.

In our knowledge, there does not seem to be any work on delayed MABCs using a nonparametric framework. In this work, we propose an algorithm accounting for delayed rewards with optimal treatment decision making as the motivation. We use nonparametric estimation to estimate the functional relationship between the rewards and the covariates. We show that the proposed algorithm is strongly consistent in that the cumulated rewards almost surely converge to the optimal cumulated rewards.

\section{Problem setup}
Assume that there are $\ell \geq 2$ arms available for allocation. Each arm allocation results in a reward which is obtained at some random time after the arm allocation. For each time $j \geq 1$, a treatment $I_j$ is alloted based on the data observed previously and the covariate $X_j$. We assume that the covariates are $d$-dimensional continuous random variables and take values in the hypercube $[0,1]^d$.  Since the rewards can be obtained at some delayed time, we denote $\{t_j \in \mathbb{R}^+, j\geq 1\}$ to be the observation time for the rewards for arms $\{I_j, j\geq 1\}$ respectively.  Let $Y_{i,j}$ be the reward obtained at time $t_j \geq j$ for arm $i = I_j$. The mean reward with covariate $X_j$ for the $i$\ts{th} arm is denoted as $f_i(X_j), 1\leq i \leq \ell$. The observed reward with covariate $X_j$ by pulling the $i$th arm is modeled as, $
Y_{i,j} = f_{i}(X_j) + \epsilon_{i,j}$,
where $\epsilon_{i,j}$ denotes random error with $\text{E}(\epsilon_{i,j}) = 0$ and $\text{Var}(\epsilon_{i,j}) < \infty$ for all $1\leq i \leq \ell$ and $j\in \mathbb{N}$.
The functions $f_i$ are assumed to be unknown and not of any given parametric form.

The rewards are observed at delayed times $t_j$; the delay in the reward for arm $I_j$ pulled at the $j$\ts{th} time is given by a random variable $d_j:= t_j - j$. Assume that these delays are mutually independent, independent of the covariates, and could be drawn from different distributions. That is, let $\{d_j, j\geq 1\}$ be a sequence of independent random variables with probability density functions $\{g_j, j \geq 1\}$ and the cumulative distribution functions $\{G_j, j \geq 1\}$, respectively. 

 Let $\{X_j,j\geq 1\}$ be a sequence of covariates independently generated according to an unknown underlying probability disribution $P_X$, from a population supported in $[0,1]^d$. 
 Let $\delta$ be a sequential allocation rule, which for each time $j$ chooses an arm $I_j$ based on the previous observations and $X_j$. The total mean reward up to time $n$ is $\sum_{j=1}^n f_{I_j}(X_j)$. 
To evaluate the performance of the allocation strategy, let $i^*(x) = \argmax_{1\leq i \leq \ell} f_i(x)$ and $f^*(x) = f_{i^*(x)}(x)$.
 Without the knowledge of the random errors, the ideal performance occurs when the choices of arms selected $I_1,\hdots, I_n$ match the optimal arms $i^*(X_1), \hdots, i^*(X_n)$, yielding the optimal total reward $\sum_{j=1}^n f^*(X_j)$.
The ratio of these two quantities is the quantity of interest,
\begin{equation}
R_n(\delta) = \dfrac{\sum_{j=1}^n f_{I_j}(X_j)}{\sum_{j=1}^n f^*(X_j)}.
\end{equation}
It can be seen that $R_n$ is a random variable no bigger than 1. 

\begin{definition} 
An allocation rule $\delta$ is said to be strongly consistent if $R_n(\delta) \rightarrow 1$ with probability 1, as $n \rightarrow \infty$.
\end{definition}

In Section \ref{algorithm}, we propose an allocation rule which takes into account reward delays. Then in Sections \ref{consistency} and \ref{consistency_histogram}, we discuss the consistency of the proposed allocation rule under some assumptions and then validate those assumptions when the histogram method is used to estimate the regression functions respectively.

\section{The proposed strategy}\label{algorithm}
Let $Z^{n,i}$ denote the set of observations for arm $i$ whose rewards have been obtained up to time $n$, that is, $Z^{n,i}:= \{(X_j,Y_{i,j}): 1\leq t_j \leq n \ \text{and} \ I_j = i\}$. Let $\hat{f}_{i,n}$ denote the regression estimator of $f_i$ based on the data $Z^{n,i}$. Let $\{\pi_j, j \geq 1 \}$ be a sequence of positive numbers in  $[0,1]$ decreasing to zero. 
\begin{enumerate}[Step 1.]
  \item \textbf{Initialize.} Allocate each arm once, w.l.o.g., we can have $I_1 = 1, I_2 = 2, \hdots, I_\ell = \ell$. Since the rewards are not immediately obtained for each of these $\ell$ arms, we continue these forced allocations until we have at least one reward observed for each arm. Suppose, that happens at time $m_0$.
  \item \textbf{Estimate the individual functions $f_i$.} For $n = m_0$, based on $Z^{n,i}$, estimate $f_i$ by $\hat{f}_{i,n}$ for $1 \leq i \leq \ell$ using the chosen regression procedure.
  \item \textbf{Estimate the best arm.} For $X_{n+1}$, let 
  $\hat{i}_{n+1}(X_{n+1}) = \arg\max_{1\leq i \leq \ell} \hat{f}_{i,n}(X_{n+1})$.
   \item \label{randomization_step}  \textbf{Select and pull.} Randomly select an arm with probability $1-(\ell-1)\pi_{n+1}$ for $i = \hat{i}_{n+1}$ and with probability $\pi_{n+1}$, for all other arms, $i \neq \hat{i}_{n+1}$. Let $I_{n+1}$ denote this selected arm.
  \item \textbf{Update the estimates.}
  \begin{enumerate}[Step 5a.]
    \item If a reward is obtained at the $(n+1)$\ts{th} time (could be one or more rewards corresponding to one or more arms $I_j, 1\leq j \leq (n+1)$), update the function estimates of $f_i$ for the respective arm (or arms) for which the reward (or rewards) are obtained at $(n+1)\ts{th}$ time.
    \item If no reward is obtained at the $(n+1)$\ts{th} time, use the previous function estimators, i.e. $\hat{f}_{i,n+1} = \hat{f}_{i,n} \ \forall \ i \in \{1,\hdots,\ell\}$. 
  \end{enumerate}
  \item \textbf{Repeat.} Repeat steps 3-5 when the next covariate $X_{n+2}$ surfaces and so on.
\end{enumerate}
The choice of $\pi_n$ in the randomization step \ref{randomization_step} is crucial in determining how much exploration and exploitation is done at any phase of the trial. To emphasize the role of $\pi_n$, we may use $\delta_{\pi}$ to denote the allocation rule. In order to select the best arm as time progresses, $\pi_n$ needs to decrease to zero but the rate of decrease will play a key role in determining how well the allocations work. For example, if in our set-up we have large delays for some arms then it might be beneficial to decrease $\pi_n$ at a slower rate so that there is enough exploration and the accuracy of our estimates is not affected in the long run. We use a user-determined choice of $\pi_n$ in this work, that is, the sequence $\pi_n$ does not adapt to the data.
  
\subsection{Consistency of the proposed strategy}\label{consistency}

Let $A_n := \{j: t_j \leq n\}$, denote the time points for which rewards were obtained by time $n$. 
 If $A_n$ is known, then the total number of observed rewards until time $n$, denoted by $N_{n}$, is also known. Recall that it is possible to observe multiple rewards at the same time point. Given $A_n$, let $\{s_k, k=1,\hdots, N_n\}$ be the reordered sequence of these observed reward timings, $\{t_k, k\in A_n\}$, arranged in a non-decreasing order. 

\textbf{Assumption 1. } 
   The regression procedure is strongly consistent in $L_\infty$ norm for all individual mean functions $f_i$ under the proposed allocation scheme. That is, $||\hat{f}_{i,n} - f_i||_\infty \overset{\text{a.s.}}{\rightarrow} 0$ as $n \rightarrow \infty$ for each $1\leq i \leq \ell$. As described in the allocation strategy in Section \ref{algorithm}, $\hat{f}_{i,n}$ is the estimator based on all previously observed rewards. That is, after initialization, the mean reward function estimators are only updated at the time points $\{s_k, k=1,\hdots N_n\}$ where $N_n$ is the number of rewards observed by time $n$. Therefore, this condition is equivalent to saying $||\hat{f}_{i,s_n} - f_i||_\infty \overset{\text{a.s.}}{\rightarrow} 0$ as $n\rightarrow \infty$.

\textbf{Assumption 2.} 
  Mean functions satisfy $f_i(x) \geq 0$, $A = \ds\sup_{1\leq i\leq \ell} \sup_{x \in [0,1]^d} (f^*(x) - f_i(x)) < \infty \ \text{and}\ \E (f^*(X_1)) > 0.$
\vspace{0.2cm}
\begin{theorem}\label{Theorem 1}
Under Assumptions 1 and 2, the allocation rule $\delta_\pi$ is strongly consistent as $n\rightarrow \infty$.
\end{theorem}
\begin{proof}
Note that consistency holds only when the sequence $\{\pi_n, n\geq 1\}$ is chosen such that $\pi_n \rightarrow 0$ as $n\rightarrow \infty$.
The proof is very similar to the proof in \cite{yang2002randomized}. The details can be found in the supplementary material (see \textit{Appendix A.1} in Appendix).
\end{proof}

Note that Assumption 1, seemingly natural, is a strong assumption and it requires additional work to verify this assumption for a particular regression setting. We verify this assumption for the histogram method in Section \ref{consistency_histogram}. On the other hand, Assumption 2 does not involve the estimation procedure and does not require any verification.

\section{The Histogram method} \label{histogram}
In this section, we explain the histogram method for the setting with delayed rewards.
Partition $[0,1]^d$ into $M = (1/h)^d$ hyper-cubes with side width $h$, assuming $h$ is chosen such that $1/h$ is an integer. For some $x \in [0,1]^d$, let $J(x)$ denote the set of time points, for which the corresponding design points observed until time $n$ fall in the same cube as $x$, say $B(x)$, and for which the corresponding rewards are observed by time $n$. Let $N(x)$ denote the size of $J(x)$. That is, let $J(x) = \{j: X_j \in B(x), t_j \leq n\}$ and $N(x) = \sum_{j=1}^n I\{X_j \in B(x), t_j \leq n\}$. Furthermore, let $\bar{J}_i(x)$ be the subset of $J(x)$ corresponding to  arm $i$ and $\bar{N}_i(x)$ is the number of such time points, that is,
$\bar{J}_i(x) = \{j\in J(x): I_j = i \}$ and $\bar{N}_i(x) = \sum_{j=1}^n I\{I_j = i, X_j \in B(x), t_j \leq n\}$. Then the histogram estimate for $f_i(x)$ is defined as,
\begin{align*}
\hat{f}_{i,n}(x) = \frac{1}{\bar{N}_i(x)} \sum_{j \in \bar{J}_i(x)}
Y_j.
\end{align*}
For the estimator to behave well, a proper choice of the bandwidth, $h = h_n$ is necessary. Although one could choose different widths $h_{i,n}$ for estimating different $f_i$'s, for simplicity, the same bandwidth $h_n$ is used in the following sections.  
For notational convenience, when the analysis is focused on a single arm, $i$ is dropped from the subscript of $\hat{f}$, $\bar{N}$ and $\bar{J}$.

Other nonparametric methods like nearest-neighbors, kernel method, spline fitting and wavelets can also be considered for estimation. Assumption 1 could be verified for these methods using the same broad approach as illustrated in the following sections for the Histogram method, along with some method specific mathematical tools and assumptions.

 \subsection{Allocation with histogram estimates} \label{consistency_histogram}
Here, we show that the histogram estimation method along with the allocation scheme described in Section \ref{algorithm}, leads to strong consistency under some reasonable conditions on random errors, design distribution, mean reward functions and delays. As already discussed in Section \ref{consistency}, we only need to verify that Assumption 1 holds for histogram method estimators. Along with Assumption 2, we make the following assumptions.


\textbf{Assumption 3.} \label{ass_design_distribution} The design distribution $P_X$ is dominated by the Lebesgue measure with a density $p(x)$ uniformly bounded above and away from 0 on $[0,1]^d$; that is, $p(x)$ satisfies $\underbar{c} \leq p(x) \leq \bar{c}$ for some positive constants $\underbar{c} < \bar{c}$.

\textbf{Assumption 4.} \label{ass_bernstein_errors} The errors satisfy a moment condition that there exists positive constants $v$ and $c$ such that, for all $m \geq 2$, the Bernstein condition is satisfied, that is, $\E|\epsilon_{ij}|^m \leq \frac{m!}{2} v^2 c^{m-2}.$

\textbf{Assumption 5.}\label{ass_delay_independence} The delays, $\{d_j, j\geq 1\}$, are independent of each other, the choice of arms and also of the covariates.

\textbf{Assumption 6.} \label{assump_delay} Let the partial sums of delay distributions satisfy,
$\sum_{j=1}^n G_j(n-j) = \Omega(n^\alpha \log^\beta{n})$ \footnote{$f(n) = \Omega{(g(n))}$ if for some positive constant $c$,$f(n) \geq cg(n)$ when $n$ is large enough}  for some  $\alpha > 0$, $\beta \in \mathbb{R}$ or for $\alpha = 0$ and $\beta > 1$.

Note that, the choice $n^\alpha \log^\beta{n}$ could be generalized to a sub-linear function $q(n)$ with a growth rate faster than $\log{n}$.

\vspace{0.3cm}
\begin{definition} \label{mod_of_continuity}
Let $x_1, x_2 \in [0,1]^d$. Then $w(h;f)$ denotes a modulus of continuity defined by,
$
w(h;f) = \sup\{|f(x_1) - f(x_2)|: |x_{1k} - x_{2k}| \leq h \ \text{for all}\ 1 \leq k \leq d\}.$

\end{definition}

\subsection{Number of observations in a small cube for histogram estimation.}\label{A.3}
  From Assumption 3 and Assumption 5, we have that for a fixed cube $B$ with side width $h_n$ at time $n$,  
$
P(X_j \in B, t_j \leq n) = P(X_j \in B)P(t_j \leq n)  \geq \underbar{c}h_n^dG_j(n-j)$.
Let $N$ be the number of observations that fall in $B$ and are observed by time $n$, that is $N = \sum_{j=1}^n I_{\{X_j \in B, t_j\leq n\}}$. It is easily seen that $N$ is a random variable with expectation $\beta \geq \sum_{j=1}^{n}\underbar{c}h_n^dG_j(n-j)$. From the extended Bernstein inequality (see \textit{Appendix A.3} in \ref{Appendix}), we have
\begin{align}
\text{P}\left(N \leq \dfrac{\underbar{c}h_n^d \sum_{j=1}^n G_j(n-j)}{2}\right) \leq \exp\left(-\dfrac{3\underbar{c}h_n^d\sum_{j=1}^n G_j(n-j)}{28} \right). \label{no_of_obs}
\end{align}

\begin{lemma} \label{lemma_theorem}
Let $\epsilon > 0$ be given. Suppose that $h$ is small enough such that $w(h;f) < \epsilon$. Then the histogram estimator $\hat{f}_n$ satisfies,
\begin{align*}
  \text{P}_{A_n,X^n}(||\hat{f}_n - f||_\infty \geq \epsilon) &\leq M \exp\left(-\dfrac{3\pi_n \min_{1\leq b \leq M} N_b}{28} \right)\\
  &\quad \quad +2M \exp\left(- \dfrac{\min_{1\leq b \leq M} N_b \pi_n^2 (\epsilon - w(h;f))^2}{8(v^2 + c(\pi_n/2)(\epsilon-w(h;f)))} \right),
\end{align*}
where the probability $P_{A_n,X^n}$ denotes conditional probability given design points $X^n = (X_1,X_2,\hdots,X_n)$ and $A_n = \{j: t_j\leq n\}$. Here, $N_b$ is the number of design points for which the rewards have been observed by time $n$ such that they fall in the $b$th small cube of the partition of the unit cube at time $n$.
\end{lemma}
\begin{proof}
The proof of Lemma \ref{lemma_theorem} is included in the supplementary materials (\textit{Appendix A.2} in Appendix).
\end{proof}

\begin{theorem}\label{Theorem2}
Suppose Assumptions 2-6 are satisfied. If for some $\alpha > 0$ and $\beta \in \mathbb{R}$ or $\alpha =0$ and $\beta > 1$, $h_n$ and $\pi_n$ are chosen to satisfy,
\begin{equation}
n^\alpha (\log{n})^{\beta -1} h_n^d \pi_n^2 \rightarrow \infty,   \label{condition_for_assumptionA}
\end{equation} 
then the allocation rule $\delta_\pi$ is strongly consistent.
\end{theorem}

\begin{proof}[\textbf{Proof of Theorem 2.}]
The histogram technique partitions the unit cube into $M = (1/h)^d$ small cubes. For each small cube $B_b, \ 1\leq b\leq M$, in the partition of the unit cube, let $N_b$ denote the number of time points, for which the corresponding design points fall in the cube $B_b$ and corresponding arm rewards are observed by time $n$. In other words, $N_b = \sum_{j=1}^n I_{\{X_j \in B_b, t_j \leq n\}}$. Using inequality \eqref{no_of_obs} we have, 
\begin{align}
 P&\left(N_b \leq \dfrac{ \underbar{c}h_n^d\sum_{j=1}^n G_j(n-j)}{2}\right) \leq \exp\left(-\dfrac{3 \underbar{c}h_n^d \sum_{j=1}^n G_j(n-j)}{28} \right) \nonumber \\ 
\Rightarrow P&\left(\min_{1\leq b \leq M}N_b \leq  \dfrac{ \underbar{c}h_n^d\sum_{j=1}^n G_j(n-j)}{2}\right) \leq M \exp\left(-\dfrac{3 \underbar{c}h_n^d \sum_{j=1}^n G_j(n-j)}{28} \right). \label{min_no_of_points}
 \end{align}
 Let $W_1,\hdots,W_n$ be Bernoulli random variables indicating whether the $i$th arm is selected $(W_j =1)$ for time point $j$, or not $(W_j =0)$. 
 Note that, conditional on the previous observations and $X_j$, the probability of $W_j =1$ is almost surely bounded below by $\pi_j \geq \pi_n$ for $1\leq j \leq n$. Let $w(h_n;f_i)$ be the modulus of continuity as in Definition \ref{mod_of_continuity}. Note that, under the continuity assumption of $f_i$, we have $w(h_n;f_i) \rightarrow 0$ as $h_n\rightarrow 0$. Thus, for any $\epsilon >0$, when $h_n$ is small enough, $\epsilon - w(h_n;f_i) \geq \epsilon/2$. 
Consider,
 \begin{align*}
 P(||\hat{f}_{i,n} - f_i||_\infty > \epsilon) &= P\left(||\hat{f}_{i,n} - f_i||_\infty > \epsilon,\min_{1\leq b \leq M}N_b \geq \dfrac{\underbar{c}h_n^d\sum_{j=1}^n G_j(n-j)}{2} \right) \\
 &\quad \quad +  P\left(||\hat{f}_{i,n} - f_i||_\infty > \epsilon,\min_{1\leq b \leq M}N_b < \dfrac{ \underbar{c}h_n^d\sum_{j=1}^n G_j(n-j)}{2} \right)\\
 &\leq \E P_{A_n,X^n}\left(||\hat{f}_{i,n} - f_i||_\infty > \epsilon, \min_{1\leq b \leq M}N_b \geq \dfrac{ \underbar{c}h_n^d\sum_{j=1}^n G_j(n-j)}{2} \right)\\
 &\quad \quad  +  P\left(\min_{1\leq b \leq M}N_b < \dfrac{ \underbar{c}h_n^d\sum_{j=1}^n G_j(n-j)}{2} \right),
 \end{align*}
 where $P_{A_n, X^n}$ denotes conditional probability given the design points until time $n$, $X^n = \{X_1,X_2,\hdots,X_n\}$ and the event, $A_n := \{j: t_j \leq n\}$. 

 From Lemma \ref{lemma_theorem}, we have that given the design points and the time points for which rewards were observed, for any $\epsilon > 0$, when $h$ is small enough,
 \begin{align*}
  \text{P}_{A_n,X^n}(||\hat{f}_n - f||_\infty \geq \epsilon) &\leq M \exp\left(-\dfrac{3\pi_n \min_{1\leq b \leq M} N_b}{28} \right)\\
  & \quad \quad +2M \exp\left(- \dfrac{\min_{1\leq b \leq M} N_b \pi_n^2 (\epsilon - w(h_n;f))^2}{8(v^2 + c(\pi_n/2)(\epsilon-w(h_n;f)))} \right).  
\end{align*}
 Using the above inequality and \eqref{min_no_of_points}, we have,
 \begin{align*}
 P(||\hat{f}_{i,n} - f_i||_\infty > \epsilon)&\leq 2M \exp\left(-\dfrac{\underbar{c}h_n^d(\sum_{j=1}^n G_j(n-j)) \pi_n^2 (\epsilon - w(h_n;f_i))^2}{16 (v^2 + c\pi_n/2(\epsilon - w(h_n;f_i)))} \right)\\ &\quad + M \exp\left(-\dfrac{3\underbar{c}h_n^d\pi_n \sum_{j=1}^n G_j(n-j)}{56} \right)+ \exp\left( -\dfrac{3 \underbar{c}h_n^d \sum_{j=1}^n G_j(n-j)}{28} \right).
 \end{align*}
 
 It can be shown that the above upper bound is summable in $n$ under the condition, 
 \begin{align}
 \dfrac{h_n^d\pi_n^2 \sum_{j=1}^nG_j(n-j)}{\log{n}} \rightarrow \infty. \label{condition_on_delaydist}
 \end{align}
It is easy to see that this follows from Assumption 6 and \eqref{condition_for_assumptionA}.

 Since $\epsilon$ is arbitrary, by the Borel-Cantelli lemma, we have that $||\hat{f}_{i,n} - f_i||_\infty \rightarrow 0$. This is true for all arms $1\leq i \leq \ell$. Hence, this completes the proof of Theorem \ref{Theorem2}.
 \end{proof}

\subsection{Effects of reward delay distributions}
As one would expect, the amount of delay in observing the rewards will have a considerable effect on the speed of sequential learning. In terms of treatment allocation, if there are substantial delays in observing patient responses for a particular treatment, the learning for that treatment will slow down and as a result the efficiency of the allocation strategy will decrease. Therefore, Assumption 6 imposes some restrictions on the delay distributions to ensure that at least a small proportion of rewards will be obtained in finite time. It is of interest to see how the delay distribution affects the rate at which $\pi_n$ and $h_n$ are allowed to decrease. This relationship can be understood by examining condition \eqref{condition_for_assumptionA} for 
Theorem \ref{Theorem2}. 

Note that Assumption 6 and \eqref{condition_for_assumptionA} in Theorem \ref{Theorem2} can be generalized to include any function $q(x)$ with at least a growth rate faster than logarithmic growth rate.  We assume 
 $\sum_{j=1}^n G_j(n-j) = \Omega\left(q(n)\right)$
 where $q(n)$ satisfies,
 $q(n)/\log(n) \rightarrow \infty$ \ as \ $n \rightarrow \infty$.
 Then it is easy to see that $h_n$ and $\pi_n$ can be chosen such that,
 \begin{align*}
 \dfrac{h_n^d \pi_n^2 q(n)}{\log(n)} \rightarrow \infty  \ \text{as} \ n \rightarrow \infty. 
 \end{align*} which implies condition \eqref{condition_on_delaydist} holds.
A possible advantage of this is that we allow a wide range of possible delay distributions with mild restrictions on the delays. Below, we consider some cases of the delay distributions and see how they effect exploration $(\pi_n)$ and bandwidth $(h_n)$ of the histogram estimator as time progresses.

 \begin{enumerate}
   \item \label{case1} In condition \eqref{condition_for_assumptionA}, $q(n) = n^\alpha \log^\beta{n}$ for $\alpha > 0$ and $\beta \in \mathbb{R}$ or $\alpha = 0$ and $\beta > 1$. Let us first consider the case when $\alpha = 0$ and $\beta > 1$, we have $q(n) = \log^\beta{n}$ for $\beta > 1$ and we want $\sum_{j=1}^n G_j(n-j) = \Omega(\log^\beta{n})$. Consider, $\pi_n = (\log{n})^{-(\beta - 1)/(2+d)}$ for $n > m_0$ and $\beta > 1$, then for \eqref{condition_on_delaydist} to hold we need the bandwith $h_n$ also to be of order $\Omega((\log{n})^{-(\beta - 1)/(2+d)})$. For example, $h_n = (\log{n})^{-(\beta - 1)/\beta(2+d)}$ would guarantee consistency. Notice that with these $\pi_n$ and $h_n$, one would spend a lot of time in exploration and the bandwidth would also decay very slowly which would effect the accuracy of the reward function estimates until $n$ is sufficiently large.

   Notice that the restriction of partial sum of probability distributions for the delays, being at least of the order $\log^\beta{n}$ gives the possibility of modeling cases with extremely large delays. For example, in clinical studies when the outcome of interest is survival time and we want to administer treatments for a disease such that the survival time is maximized. 
   With the unprecedented advances in drug development, the life expectancy of patients is more likely to increase, hence the survival time for a patient given any treatment would be large. Therefore, the assumption that partial sums of probability distributions for the delays until time $n$ need only be at least $\log^\beta{n}$ seems to be quite reasonable when the expected waiting times (in this case survival times) are long. For example, diseases like diabetes and hypertension which have a long survival time, since they cannot be cured, but can be controlled with medications. These diseases also have fairly high prevalence, so a large sample size to be able to get close to optimality would not be a problem. For such diseases, assuming that one would only observe the responses (survival times) of a small fraction of patients in finite time seems reasonable.

   \item For the case when $\alpha > 0$ and $\beta \in \mathbb{R}$, we have that $\sum_{j=1}^n G_j(n-j) = \Omega(n^\alpha\log^\beta{n})$. Consider, $\pi_n = n^{-{\alpha}/{(2+d)}}$ for $n > m_0$, then for the condition \eqref{condition_on_delaydist} to hold we need $h_n$ to also be of order $\Omega(n^{-{\alpha}/{(2 + d)}})$. For example, $h_n = n^{-{\alpha}/{2(2+d)}}$ results in 
   $
   h_n^d \pi_n^2 n^\alpha \log^{\beta - 1}{(n)} = n^{{\alpha d}/{2(2 + d)}} \log^{\beta - 1}{(n)} \rightarrow  \infty \ \text{as n $\rightarrow \infty$}$, irrespective of the value of $\beta$. Here the lower bound on the partial sums of probability distributions for the delays can grow faster than the previous case, depending on the values of $\alpha$ and $\beta$. 

   This restriction of order $n^\alpha (\log^\beta{n})$ can model cases with moderately large delays. From a clinical point of view, one could model diseases in which treatments show their effect in a short to moderate duration of time, for examples diseases like diarrhea, common cold, headache, and nutritional deficiencies. Here the response of interest would be improvement in the condition of a patient as a result of a treatment. For such diseases, one can expect to see the treatment
effects on patients in a short period of time. Hence, the delay in observing treatment results will not be too long. If the
response considered was survival (survived or not), then stroke
could also fall in this category because of high mortality.

 Note that, Assumption 6 only restricts on the proportion of rewards expected to be observed in the long run. Therefore, it is possible for strong consistency to be achieved even when there is infinite delay in observing the rewards of some arms (non-observance of some rewards).

    \end{enumerate}

\section{Simulation study}
We conduct a simulation study to compare the effect of different delay scenarios on the per-round average regret of our proposed strategy. The per-round regret is given by,
$
r_n(\delta) = \frac{1}{n} \sum_{j=1}^n (f^*(X_j) - f_{I_j}(X_j))$. 

Note that if $\frac{1}{n} \sum_{j=1}^n f^*(X_j)$ is eventually bounded above and away from 0 with probability 1, then $R_n(\delta) \rightarrow 1$ a.s. is equivalent to $r_n(\delta) \rightarrow 0$ a.s.
\begin{figure}[h!]
\centering
 \includegraphics[width=.40\textwidth]{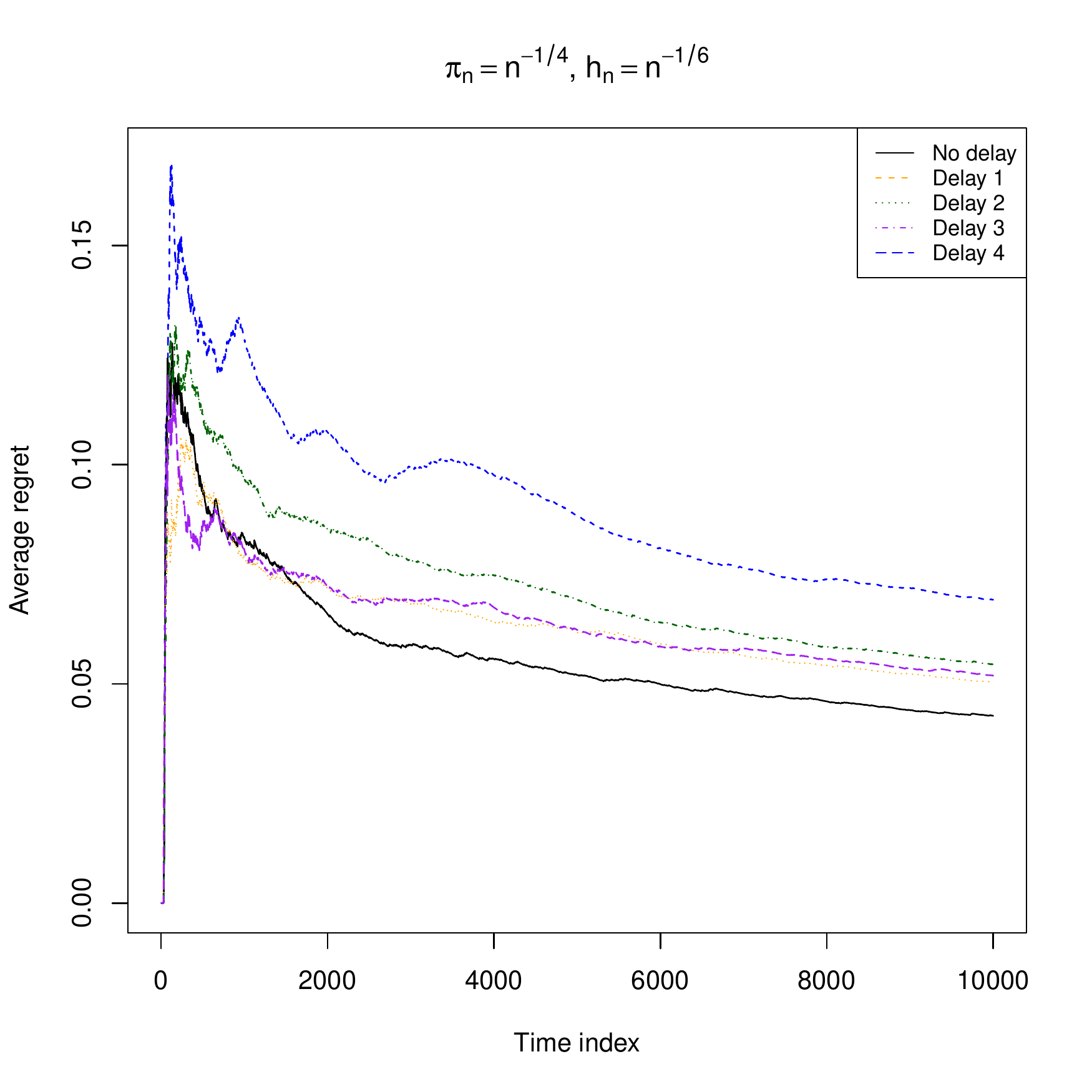}
  \includegraphics[width=.40\textwidth]{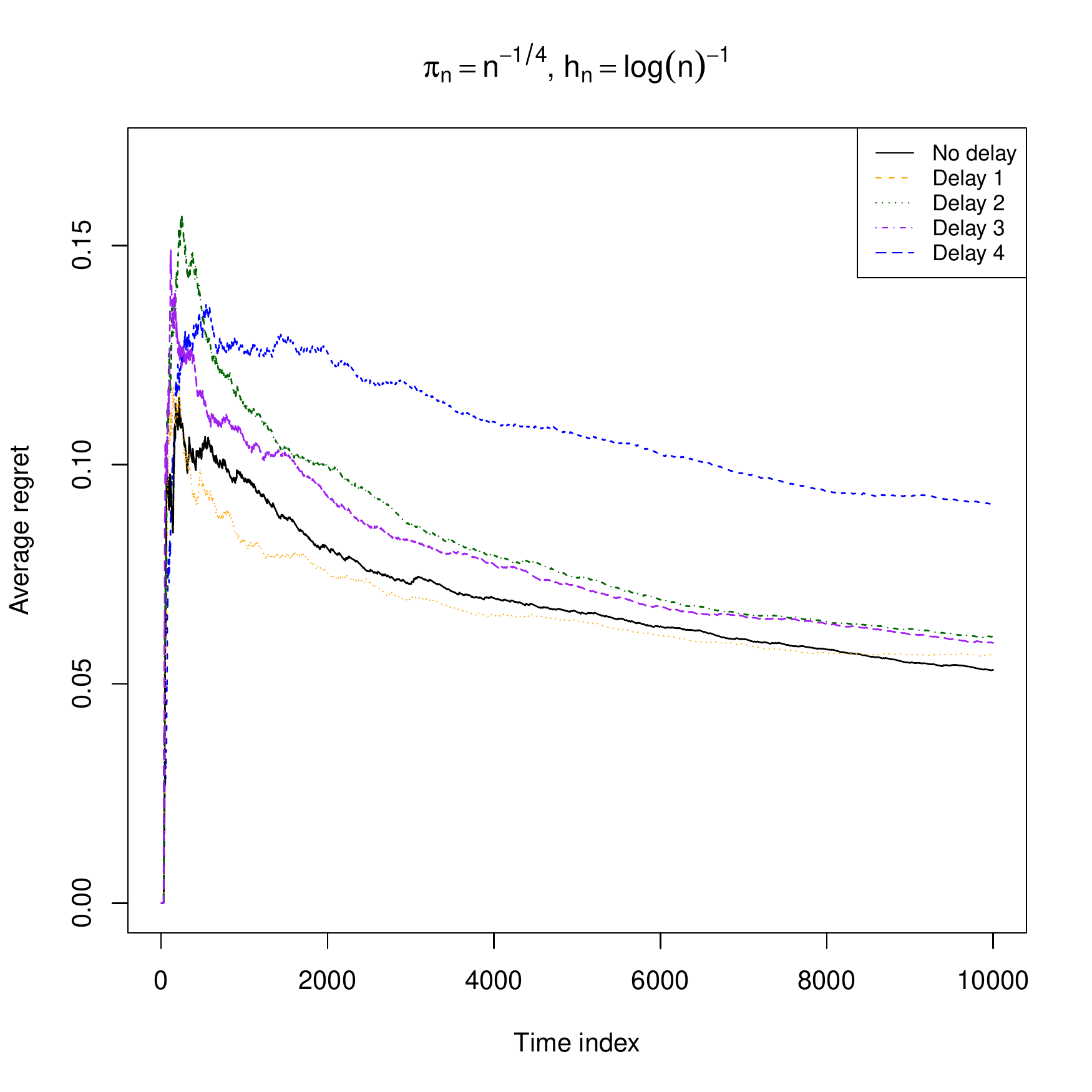}\\
\includegraphics[width=.40\textwidth]{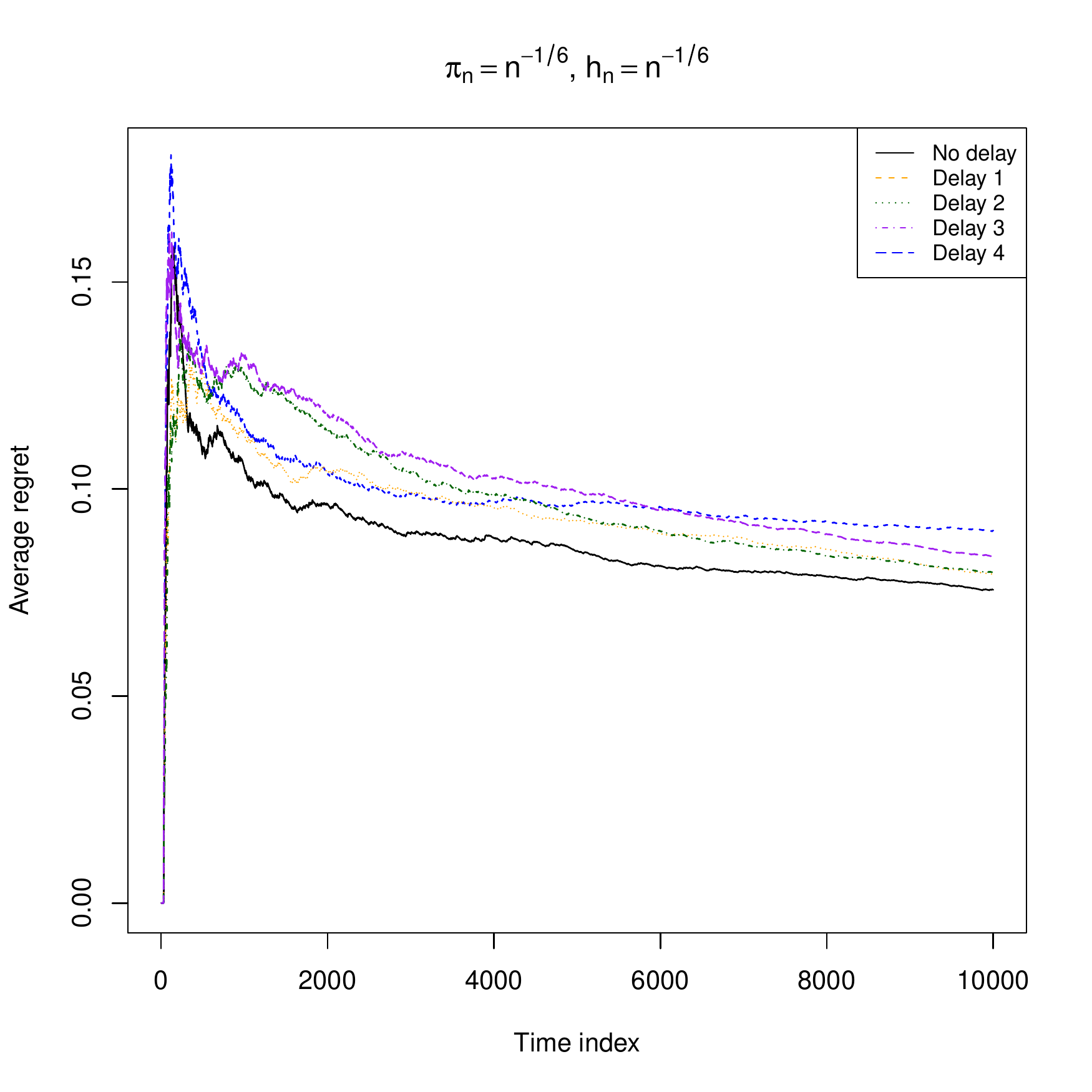}
\includegraphics[width=.40\textwidth]{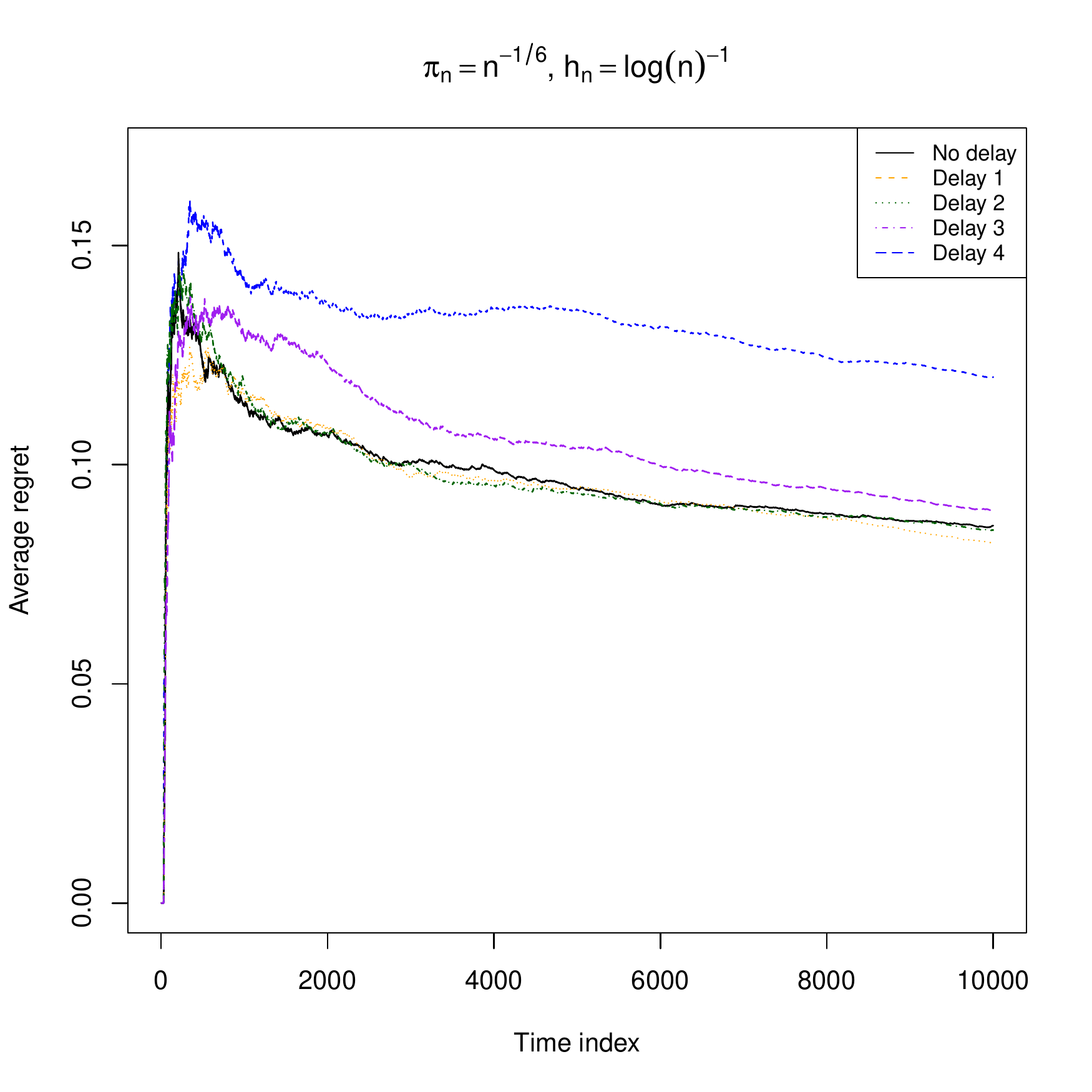}\\

\caption{Per-round regret for the proposed strategy for different delay scenarios. The grid of plots represent 4 different combination of choices for $\{\pi_n\}$ and $\{h_n\}$. For a given row, $\pi_n$ remains fixed and $h_n$ varies and vice versa for columns.}
\label{fig:simulations}

\end{figure}
\subsection{Simulation setup}
Consider number of arms, $\ell = 3$, and the covariate space to be two-dimensional, $d= 2$. Let $X_n = (X_{n1}, X_{n2})$ where $X_{ni} \overset{i.i.d.}{\sim} $ Unif$(0,1)$. We assume that the errors $\epsilon_n \sim 0.5 $N(0,1). The first 30 rounds were used for initialization. The following true mean reward functions are used, \begin{align*}
f_1(\mathbf{x}) = 0.7 (x_1 + x_2),\ 
f_2(\mathbf{x}) = 0.5 x_1^{0.75} + \sin(x_2), \ 
f_3(\mathbf{x})= \frac{2 x_1}{0.5 + (1.5 + x_2)^{1.5}}.
\end{align*}

We consider the following delay scenarios and run simulations until $N = 10000$.
1) \textit{No delay}; 2) \textit{Delay 1:} Geometric delay with probability of success (observing the reward) $p = 0.3$; 3) \textit{Delay 2:} Every 5\ts{th} reward is not observed by time $N$ and other rewards are obtained with a geometric ($p = 0.3$) delay;  4) \textit{Delay 3:} Each case has probability 0.7 to delay and the delay is half-normal with scale parameter, $\sigma = 1500$; 5) \textit{Delay 4:} In this case we increase the number of non-observed rewards. Divide the data into four equal consecutive parts (quarters), such that, in part 1, we only observe every 10\ts{th} (with Geom(0.3) delay) observation by time $N$ and not observe the remaining; in part 2, we only observe every 15\ts{th} observation; in part 3, only observe every 20\ts{th} observation; in part 4, only observe every 25\ts{th} observation. 

In Figure \ref{fig:simulations}, we plot the per-round regret vs time by delay type for four combinations of $\pi_n$ and $h_n$. 
As one would expect (see Figure \ref{fig:simulations}), the severity of delay has a clear effect on the regret, and for delay scenarios where a large number of rewards are not observed in finite time, the regret is comparatively higher. Note that most delay scenarios for which a substantial number of rewards can be obtained in finite time, tend to converge in quite similar patterns.

\textbf{Choice of $\{\pi_n\}$ and $\{h_n\}$: } 
According to Theorem 2, if $\pi_n$ and $h_n$ are chosen such that condition \eqref{condition_for_assumptionA} is met, consistency of the allocation rule follows. Therefore, for the case with $d=2$, which is the case of the simulation setting, we have to choose sequences slower than ($\pi_n = n^{-1/2}, h_n = n^{-1/2}$), even in the case of no delays.
Keeping this in mind, we chose two different choices of sequences for $\pi_n$ ($n^{-1/4}, n^{-1/6}$) and two choices of $h_n ((\log{n})^{-1}, n^{-1/6})$. 
Note that, in Figure \ref{fig:simulations}, for a given row, $\pi_n$ remains fixed while $h_n$ varies and vice versa for columns. It can be seen that the regret gets worse when $h_n$ decays too fast (in our range of n as $N = 10000$), specially for the scenario (Delay 4) with increasing number of non-observed rewards, possibly because of violation of condition \eqref{condition_for_assumptionA}. Also notice that, slow decaying $\pi_n$ has higher regret (last row).  This could be because of large randomization error that leads to high exploration price. In general, there are a large pool of choices for $h_n$ and $\pi_n$ that satisfy equation \eqref{condition_for_assumptionA} as can be seen from the Figure \ref{fig:simulations}. However, a thorough understanding of the finite-time regret rates and further research would be needed to evaluate optimal choices of $\{\pi_n\}$ and $\{h_n\}$ for a given scenario.

\section{Conclusion}
In this work we develop an allocation rule for multi-armed bandit problem with covariates when there is delay in observing rewards. We show that strong consistency can be established for the proposed allocation rule using the histogram method for estimation, under reasonable restrictions on the delay distributions and also illustrate that using a simulation study. Our approach on modeling reward delays is different from the previous work done in this field because, 1) we use nonparametric estimation technique to estimate the functional relationship between the rewards and covariates and 2) we allow for delays to be unbounded with some assumptions on the delay distributions. The assumptions impose mild restrictions on the delays in the sense that they allow for the possibility of non-observance of some rewards as long as a certain proportion of rewards are obtained in finite time. With this general setup, it is possible to model many different situations including the one with no delays. The conditions on the delay distributions easily allow for large delays as long as they grow at a certain minimal rate. This obviously will result in slower rate of convergence because of longer time spent in exploration. Ideally, we would like our allocation scheme to devise the optimal treatments sooner, for which we would need to impose stricter conditions on the delay distributions. Therefore, working on finite-time analysis for the setting with delayed rewards seems to be an immediate future direction. In addition, we assume some knowledge on the delay distributions, so for situations where there is little understanding of the delays, a different approach might be needed, such as a methodology which adaptively updates the delay distributions.

\appendix 
\section{Appendix}\label{Appendix}
\subsection{\bf Proof of consistency of the proposed strategy}\label{Proof_consistency_YangZhu}
\begin{proof}[Proof of Theorem 1]
Since the ratio $R_n(\delta_\pi){}$ is always upper bounded by 1, we only need to work on the lower bound direction. Note that,
\begin{align*}
R_n(\delta_\pi) &= \dfrac{\sum_{j=1}^n f_{\hat{i}_j}(X_j)}{\sum_{j=1}^n f^*(X_j)} + \dfrac{\sum_{j=1}^n (f_{I_j}(X_j) - f_{\hat{i}_j}(X_j))}{\sum_{j=1}^n f^*(X_j)}\\
&\geq \dfrac{\sum_{j=1}^n f_{\hat{i}_j}(X_j)}{\sum_{j=1}^n f^*(X_j)} - \dfrac{\frac{1}{n} \sum_{j=1}^n A I_{\{I_j \neq \hat{i}_j\}}}{\frac{1}{n} \sum_{j=1}^n f^*(X_j)}, \label{inequality_begin}
\end{align*}
where the inequality follows from Assumption 2. Let $U_j = I_{\{I_j \neq \hat{i}_j\}}$. Since $(1/n)\sum_{j=1}^n f^*(X_j)$ converges a.s. to $\E f^*(X) > 0$, the second term on the right hand side in the above inequality converges to zero almost surely if $({1}/{n}) \sum_{j=1}^n U_j \overset{\text{a.s.}}{\rightarrow} 0$.
Note that for $j \geq m_0 +1$, $U_j$'s are independent Bernoulli random variables with success probability $(\ell-1)\pi_j$.
Since,
\begin{align*}
\sum_{j=m_0+1}^\infty \text{Var} \left(\dfrac{U_j}{j}\right) = \sum_{j=m_0+1}^{\infty} \dfrac{(\ell - 1)\pi_j (1-(\ell-1)\pi_j)}{j^2} < \infty.
\end{align*}
we have that $\sum_{m_0+1}^\infty ((U_j - (\ell -1)\pi_j)/j)$ converges almost surely. It then follows by Kronecker's lemma that,
\begin{align*}
\dfrac{1}{n} \sum_{j=1}^n (U_j - (\ell-1)\pi_j) \overset{\text{a.s.}}{\rightarrow} 0.
\end{align*}
We know that $\pi_j \rightarrow 0$ as $j \rightarrow \infty$ (the speed depending on the delay times). Thus, we will have ${1}/{n} \sum_{j=1}^n (\ell-1)\pi_j \rightarrow 0$ since $\pi_j\rightarrow 0$ as $j\rightarrow \infty$. Hence, ${1}/{n} \sum_{j=1}^n U_j\rightarrow 0$ a.s.

To show that $R_n(\delta_\pi) \overset{\text{a.s.}}{\rightarrow} 1$, it remains to show that
\begin{align*}
 \dfrac{\sum_{j=1}^n f_{\hat{i}_j}(X_j)}{\sum_{j=1}^n f^*(X_j)} \overset{\text{a.s.}}{\rightarrow} 1 \ \text{or equivalently,}\ \dfrac{\sum_{j=1}^n (f_{\hat{i}_j}(X_j) - f^*(X_j))}{\sum_{j=1}^n f^*(X_j)} \overset{\text{a.s.}}{\rightarrow} 0.
 \end{align*} 
Recall from Section 3.1, 
given the observed reward timings $\{t_j: t_j \leq n , 1\leq j \leq n\}$, let $\{s_k: k =1,\hdots, N_n\}$ be the reordered sequence of the observed reward timings, arranged in an increasing order. Then for any $j, m_0+1\leq j \leq n$, there exists an $s_{k_j}, k_j\in \{1,\hdots, N_n\}$ such that $s_{k_j} \leq j < s_{k_j+1}$. Also, note that as $j \rightarrow \infty$, we also have that $k_j \rightarrow \infty$. By the definition of $\hat{i}_j$, for $j \geq m_0 + 1$, $\hat{f}_{\hat{i}_j,s_{k_j}}(X_j) \geq \hat{f}_{i^*(X_j),s_{k_j}}(X_j)$ and thus,
\begin{align*}
f_{\hat{i}_j}(X_j) - f^*(X_j) &= f_{\hat{i}_j}(X_j) - \hat{f}_{\hat{i}_j,s_{k_j}}(X_j) + \hat{f}_{\hat{i}_j, s_{k_j}}(X_j) - \hat{f}_{i^*(X_j),s_{k_j}}(X_j)\\
&\quad \quad \quad + \hat{f}_{i^*(X_j),s_{k_j}}(X_j) - f^*(X_j)\\
&\geq f_{\hat{i}_j}(X_j) - \hat{f}_{\hat{i}_j, s_{k_j}}(X_j) + \hat{f}_{i^*(X_j),s_{k_j}}(X_j) - f_{i^*(X_j)}(X_j)\\
&\geq -2 \sup_{1\leq i \leq \ell} ||\hat{f}_{i,s_{k_j}} - f_i||_\infty.
\end{align*}
For $1\leq j \leq m_0$, we have $f_{\hat{i}_j}(X_j) - f^*(X_j) \geq -A$.  Based on Assumption A, $||\hat{f}_{i,s_{k_j}} - f_i||_\infty \overset{\text{a.s.}}{\rightarrow} 0$ as $j \rightarrow \infty$ for each $i$, and thus $\sup_{1\leq i \leq \ell} || \hat{f}_{i,s_{k_j}} - f_i||_\infty \overset{\text{a.s.}}{\rightarrow} 0$. Then it follows that, for $n > m_0$,
\begin{align*}
&\dfrac{\sum_{j=1}^n (f_{\hat{i}_j}(X_j) - f^*(X_j))}{\sum_{j=1}^n f^*(X_j)} \\
&\quad \quad \geq \dfrac{-Am_0/n - (2/n)\sum_{j=m_0+1}^n \sup_{1\leq i \leq \ell} ||\hat{f}_{i,s_{k_j}} - f_i||_\infty}{(1/n)\sum_{j=1}^n f^*(X_j)}.
\end{align*}
The right hand side converges to 0 almost surely and hence the conclusion follows. 
\end{proof}

\subsection{\bf A probability bound on the performance of the histogram method}\label{A.1}
Consider the regression model as in Section 2, with $i$ dropped for notational convenience.
\begin{align*}
Y_{j} = f(x_j) + \epsilon_{j},
\end{align*}
where $\epsilon_j$'s are independent errors satisfying the moment condition in Assumption 4 of Section 4.1. Let $W_1,\hdots, W_n$ are Bernoulli random variables that decide if arm $i$ is observed or not, that is $W_j = I_{\{I_j = i\}}$. Assume, for each $1\leq j \leq n$, $W_j$ is independent of $\{\epsilon_k: k\geq j\}$. Let $\hat{f}_n$ be the histogram estimator of $f$. Let $w(h;f)$ denote a modulus of continuity defined by,
\begin{equation}
  w(h;f) = \sup\{|f(x_1) - f(x_2)|: |x_{1k} -x_{2k}| \leq h \ \text{for all} \ 1\leq k \leq d\}.
\end{equation}
Let $A_n$ denote the event consisting of the indices (time points) for which the rewards were observed by time $n$, that is $A_n := \{j: t_j \leq n\}$ and $X^n = \{X_1,X_2,\hdots,X_n\}$, the design points until time $n$.

\begin{lemma} \label{lemma1_appendix}
Let $\epsilon > 0$ be given. Suppose that $h$ is small enough that $w(h;f) < \epsilon$. Then the histogram estimator $\hat{f}_n$ satisfies,
\begin{align*}
  \text{P}_{A_n,X^n}(||\hat{f}_n - f||_\infty \geq \epsilon) &\leq M \exp\left(-\dfrac{3\pi_n \min_{1\leq b \leq M} N_b}{28} \right) \\
  & \quad \quad +2M \exp\left(- \dfrac{\min_{1\leq b \leq M} N_b \pi_n^2 (\epsilon - w(h;f))^2}{8(v^2 + c(\pi_n/2)(\epsilon-w(h;f)))} \right).
\end{align*}
where the probability $P_{A_n,X^n}$ denotes conditional probability given design points $X^n = (X_1,X_2,\hdots,X_n)$ and $A_n = \{j: t_j\leq n\}$. Here, $N_b$ is the number of design points for which the rewards have been obtained by time $n$ such that they fall in the $b$th small cube of the partition of the unit cube at time $n$.
\end{lemma}

\begin{proof}[Proof of lemma \ref{lemma1_appendix}]
Note that the above inequality trivially holds if $\min_{1\leq b \leq M} N_b = 0$. Therefore,
let's assume that $\min_{1\leq b\leq M} N_b > 0$. Let $N(x)$ denote the number of time points, for which the corresponding design points $x_j$'s fall in the same cube as $x$ and for which the corresponding rewards are observed by time $n$. Let $J(x)$ denote the set of indices $1\leq j \leq n$ of such design points. Let $\bar{J}(x)$ be the subset of $J(x)$ where arm $i$ is chosen (i.e. where $W_j =1$) and let $\bar{N}(x)$ be the number of such design points (note that $i$ is dropped for notational convenience). 

For arm $i$, we consider the histogram estimator
\begin{align*}
\hat{f}_{n}(x) &= \dfrac{1}{\bar{N}(x)}\sum_{j \in \bar{J}(x)} Y_j \\
& = f(x) + \dfrac{1}{\bar{N}(x)}\sum_{j \in \bar{J}(x)} (f(x_j) - f(x)) + \dfrac{1}{\bar{N}(x)}\sum_{j \in\bar{J}(x)} \epsilon_j\\
\Rightarrow |\hat{f}_n(x)& - f(x)| \leq w(h;f) + \left|\dfrac{1}{\bar{N}(x)}\sum_{j \in \bar{J}(x)}\epsilon_j \right|,
 \end{align*}
 where $w(h;f)$ is the modulus of continuity. 
 For any $\epsilon > w(h;f)$, with the given design points and the time points for which rewards have been observed by time $n$,
 \begin{align*}
 P_{A_n,X^n}(||\hat{f}_n - f||_\infty \geq \epsilon) \leq P_{A_n,X^n}\left(\sup_x \left|\dfrac{1}{\bar{N}(x)}\sum_{j \in \bar{J}(x)} \epsilon_j \right| \geq \epsilon  -w(h;f) \right). 
 \end{align*}
 Note that, in the same small cube $B$, $N(x)\ \text{and}\ \bar{N}(x), J(x)\ \text{and}\ \bar{J}(x)$ are the same for any $x$, respectively. Let $x_0$ be a fixed point in $B$. Then consider,
 \begin{align*}
 P_{A_n,X^n}&\left(\sup_{x \in B} \left|\dfrac{1}{\bar{N}(x)} \sum_{j\in \bar{J}(x)} \epsilon_j \right| \geq \epsilon- w(h;f)\right)\\
& = P_{A_n,X^n}\left(\left| \sum_{j\in \bar{J}(x_0)} \epsilon_j \right| \geq \bar{N}(x_0) (\epsilon- w(h;f)) \right)\\
& = P_{A_n,X^n}\left(\left| \sum_{j\in J(x_0)} W_j \epsilon_j \right| \geq N(x_0) \dfrac{\bar{N}(x_0)}{N(x_0)} (\epsilon- w(h;f)) \right)\\
& = P_{A_n,X^n}\left(\left| \sum_{j\in J(x_0)} W_j \epsilon_j \right| \geq N(x_0) \dfrac{\bar{N}(x_0)}{N(x_0)} (\epsilon- w(h;f)), \dfrac{\bar{N}(x_0)}{N(x_0)} > \dfrac{\pi_n}{2} \right) \\
&\quad \quad+ P_{A_n,X^n}\left(\left| \sum_{j\in J(x_0)} W_j \epsilon_j \right| \geq N(x_0) \dfrac{\bar{N}(x_0)}{N(x_0)} (\epsilon- w(h;f)), \dfrac{\bar{N}(x_0)}{N(x_0)} \leq \dfrac{\pi_n}{2} \right) \\
& \leq P_{A_n,X^n}\left(\left| \sum_{j\in J(x_0)} W_j \epsilon_j \right| \geq N(x_0) \dfrac{\pi_n}{2} (\epsilon- w(h;f)) \right) + P_{A_n,X^n}\left(\dfrac{\bar{N}(x_0)}{N(x_0)} \leq \dfrac{\pi_n}{2} \right)\\
& \leq 2 \exp\left(- \dfrac{N(x_0) \pi_n^2 (\epsilon - w(h;f))^2}{8(v^2 + c\pi_n/2 (\epsilon - w(h;f)))} \right) + \exp\left(- \dfrac{3 N(x_0) \pi_n}{28} \right),
\end{align*}
where the last inequality follows from inequality \eqref{bernstein_inequality} in \ref{A.4} and \eqref{binomial_inequality} in \ref{A.2} respectively. For applying \eqref{bernstein_inequality}, we used the fact that $W_j$ is independent of the $\epsilon_{ik}$'s for all $k\geq j$ since $W_j$ depends only on the previous observations and $X_j$.

 Given that $N_b$ be the number of design points in the $b$th small cube whose rewards are observed by time $n$, we have
 \begin{align*}
 P_{A_n,X^n}(||\hat{f}_n - f||_\infty) &\leq M \exp\left(- \dfrac{3 (\min_{1\leq b \leq M}N_b) \pi_n}{28} \right) \\
 & \quad \quad + 2M \exp\left(- \dfrac{(\min_{1\leq b \leq M}N_b) \pi_n^2 (\epsilon - w(h;f))^2}{8(v^2 + c(\pi_n/2) (\epsilon - w(h;f)))} \right).
 \end{align*}
 This concludes the proof of Lemma 1.
\end{proof}

\subsection{\bf An inequality for Bernoulli trials.}\label{A.2}
For $1\leq j \leq n$, let $\tilde{W}_j$ be Bernoulli random variables, which are not necessarily independent. Assume that the conditional probability of success for $\tilde{W}_j$ given the previous observations is lower bounded by $\beta_j$, that is,
\begin{align*}
P(\tilde{W}_j = 1|\tilde{W}_i, 1 \leq i \leq j-1) \geq \beta_j \ \text{a.s.},
\end{align*}
for all $1\leq j\leq n$. 
Appylying the extended Bernstein's inequality as described in \cite{qian2016kernel}, we have
\begin{align}
P\left(\sum_{j=1}^n \tilde{W}_j \leq \left(\sum_{j=1}^n \beta_j\right)/2 \right) \leq \exp\left(- \dfrac{3\sum_{j=1}^n \beta_j}{28} \right).  \label{binomial_inequality}
\end{align}


\subsection{\bf A probability inequality for sums of certain random variables.} \label{A.4} 
 Let $\epsilon_1,\epsilon_2,\hdots$ be independent random variables satisfying the refined Bernstein condition in Assumption 3. Let $I_1,I_2,\hdots$ be Bernoulli random variables such that $I_j$ is independent of $\{\epsilon_l: l\geq j\}$ for all $j \geq 1$.

\begin{lemma}
For any $\epsilon >0$,
\begin{align}
P\left(\sum_{j=1}^n I_j \epsilon_j \geq n\epsilon\right) \leq \exp\left(-\dfrac{n\epsilon^2}{v^2 + c\epsilon}\right).\label{bernstein_inequality}
\end{align}
\end{lemma}

The proof for this lemma can be found in \cite{yang2002randomized}.

\bibliography{mybibfile_rev}

\end{document}